%% file: LargeScaleHawkes_arxiv.tex
\documentclass{article} 
\usepackage[margin=1.5in]{geometry}

\input{definitions.tex}

\newcommand{\titleName}   	 {Multivariate Hawkes Processes for Large-scale Inference}
\newcommand{\methodName}     {Low-Rank Hawkes Processes\xspace}
\newcommand{\methodNameBis}    {Low-Rank Hawkes Process\xspace}
\newcommand{\methodInitials} {LRHP\xspace}
\newcommand{\mhp} 					 {MHP\xspace}

\title{\titleName}

\author{
R\'emi Lemonnier$^{1,2}$
\hspace{1.5em}
Kevin Scaman$^1$
\hspace{1.5em}
Argyris Kalogeratos$^1$
\\$^1$ CMLA {--} ENS Cachan, CNRS, Universit\'e Paris-Saclay, France
\\$^2$ Numberly, 1000Mercis group, Paris, France
\\\texttt{\{lemonnier, scaman, kalogeratos\}@cmla.ens-cachan.fr}
}
\date{}

\begin{document}

\maketitle
%
\begin{abstract}
In this paper, we present a framework for fitting multivariate Hawkes processes for large-scale problems both in the number of events in the observed history $n$ and the number of event types $d$ (\ie dimensions). The proposed \emph{\methodNameBis} (\methodInitials) framework introduces a low-rank approximation of the kernel matrix that allows to perform the nonparametric learning of the $d^2$ triggering kernels using at most $O(ndr^2)$ operations, where $r$ is the rank of the approximation ($r \op{\ll} d,n$). This comes as a major improvement to the existing state-of-the-art inference algorithms that are in $O(nd^2)$.
Furthermore, the low-rank approximation allows \methodInitials to learn representative patterns of interaction between event types, which may be valuable for the analysis of such complex processes in real world datasets.
The efficiency and scalability of our approach is illustrated with numerical experiments on simulated as well as real datasets.
%
\end{abstract}
%
%
%
\section{Introduction}\label{sec:intro}
%

In many real-world phenomena, such as product adoption or information sharing, events exhibit a \emph{mutually-exciting} behavior, in the sense that the occurrence of an event will increase the occurrence rate of others.
In the field of internet marketing, the purchasing behavior of a client of an online shopping website can be, to a large extent, predicted by his past navigation history on other websites.
In finance, arrivals of buying and selling orders for different stocks convey information about macroscopic market tendencies.
In the study of information propagation, users of a social network share information from one to another, leading to \emph{information cascades} spreading throughout the social graph.
Over the past few years, the study of point processes gained attention as the acquisition of such datasets by companies and research laboratories became increasing simple. 
However, the traditional models for time series analysis, such as discrete-time auto-regressive models, do not apply in this context due to fact that events happen in a continuous way.

Multivariate Hawkes processes (MHP) \cite{oakes1975markovian,liniger2009multivariate} have emerged in several fields as the gold standard to deal with such data, \eg earthquake prediction \cite{vere1978earthquake}, biology \cite{reynaud2013spike}, financial \cite{bauwens2009modelling,alfonsi2014dynamic} and social interactions studies \cite{crane2008robust}. For MHP, an event of type $u$ (\eg the visit of a product's website) occurring at time $t$, will increase the conditional rate of occurrence of events of type $v$ at time $s \op {\geq} t$ (\eg purchases of  this product in the future) by a rate $g_{uv}(s \op{-} t)$.

While these processes have been extensively studied from the probabilistic point of view (stability \cite{bremaud1996stability}, cluster representation \cite{bordenave2007large}), their application to real-scale datasets remains quite challenging. For instance, social interactions data is at the same time \emph{big} (large number of posts), \emph{high-dimensional} (large number of users), and \emph{structured} (social network). 

Several nonparametric estimation procedures have been proposed for \mhp \cite{hansen2012lasso,bacry2013modelling,conf/icml/ZhouZS13}. However, due to the dependence of the stochastic rate of occurrence at a given time on all past occurrences, these estimation procedures are quadratic in the number of events, which renders them impractical for large datasets. In the direction of tackling this issue, \cite{lemonnier2014nonparametric} proposed a nonparametric linear-time estimation procedure relying on the \emph{memoryless property} of Hawkes processes with exponential triggering kernels. 
However, the complexity of their algorithm remains quadratic in the number of dimensions, since each of the $d^2$ triggering kernels $g_{uv}$ needs to be estimated.

In this paper we introduce \emph{\methodName} (\methodInitials), a model for structured point processes relying on a \emph{low-rank decomposition of the triggering kernel} that aim to learn representative patterns of interaction between event types. We also provide an efficient and scalable inference algorithm for \methodInitials with linear complexity in the total number of events and number of event types (\ie dimensions). This inference is performed by combining minorize-maximization and self-concordant optimization techniques. In addition, if the underlying network of interactions is provided, the presented algorithm fully exploits the network's sparsity, which makes it practical for large and structured datasets. The major advantage of the the proposed \methodInitials algorithm is that it is able to scale-up to graphs much larger than previous state-of-the-art methods, while maintaining performances very close to state-of-the-art competitors in terms of prediction and inference accuracy on synthetic as well as real datasets.

The rest of the paper is as follows. In Section 2, we recall the definition of \mhp and introduce the associated inference problem. In Section 3, we project the original dimensions in a low-rank space and decompose the triggering functions over a basis of exponential kernels. In Section 4, we develop our new inference algorithm  \methodInitials and show that its theoretical complexity is lower than state-of-the-art. In Section 5, we empirically prove that \methodInitials outperforms significantly the state-of-the-art in terms of computational efficiency, while maintaining a very high level of precision for the task of recovering triggering kernels as well as predicting future events.

\begin{table*}[t] \label{tab:notations} 
\hspace{-1.2em}
\scriptsize 
\begin{tabular}{ l | l }
  \toprule
	\textbf{Symbol} & \textbf{Description}\\
	\midrule
	$d$ & number of event types, \ie dimensions of the multivariate Hawkes process\\
	$r$ & rank of the low-dimensional approximation\\
	$n$ & number of events of all realizations of the \methodInitials process\\
	$K$ & number of triggering kernels\\
	\midrule
	$G \op{=} \{\mathcal{V},\mathcal{E}\}$ & a network of $d$ nodes, node set $\mathcal{V}$ and edge set $\mathcal{E}$\\
	$A$ 							& network's adjacency matrix \\
    $\Delta$ 				& maximum node degree of $G$ \\
	\midrule
  $u,v \op{=} 1,\mydots,d$ & indices on dimensions of the original space \\
	$i,j\op{=}1,\mydots,r$ & indices on dimensions of the low-dimensional embedding\\
	$P$ 						  & $d\times r$ event type-to-group projection matrix \\
	\midrule
	$N(t) \op{=} [N_u(t)]_u$ & $d$-dimensional counting process ($t \op{\geq} 0, u \op{=} 1,\mydots,d$)\\
	$\lambda_u(t)$    & non-negative occurrence rate for event type $u$ at time $t$ \\
	$\mu_u(t)$        & natural occurrence rate for event type $u$ at time $t$ \\
	$g_{vu}(\Delta t)$& kernel function evaluating the affection of $\lambda_u$ due to events of type $v$ at time distance $\Delta t$ \\
	\midrule
	$\alpha, \beta$   & parameters of the triggering kernels \\
	$\gamma, \delta$  & hyperparameters of the triggering kernels \\
	\midrule
	$h \op{=} 1,\mydots,H$   & realizations of the \methodInitials process ($d$-dimensional) \\
	$m \op{=} 1,\mydots,n_h$ & events of the realization $h$, which may belong to any event type \\
	$\mathcal{H}^h$             & history of $(t_m^h, u_m^h)_{m=1}^{n_h}$ events of the realization $h$, indicating (time of event, event type) \\
	$\mathcal{H}$               & collection of the event histories of all $H$ realizations \\
	$\sigma$          & maximum number of event types involved in a realization \\
	
  \midrule
	$B, D$          & tensors with four and five dimensions, respectively, introduced to simplify our inference algorithm \\
	\bottomrule
\end{tabular}
\caption{Index of main notations.}
\end{table*}
%
%
%
\section{Setup and Notations}\label{sec:notations}
%
A multivariate Hawkes process (\mhp) $N(t) \op{=} \{N_u(t)\!: u \op{=} 1, \mydots, d, t \op{\geq} 0\}$ is a $d$-dimensional counting process, where $N_u(t)$ represents the number of events along dimension $u$ that occurred during time $[0,t]$. We will call \emph{event of type $u$} an event that occurs along dimension $u$. Each one-dimensional counting process $N_u(t)$ can be influenced by the occurrence of events of other types. Without loss of generality, we will consider that these \emph{mutual excitations} take place along the edges of an unweighted directed network $\mathcal{G} \op{=} (\mathcal{V}, \mathcal{E})$ of $d$ nodes and adjacency matrix $A \op{\in} \matdef{\{0,1\}}{d}{d}$. Finally, we denote as $\mathcal{H}:(u_m, t_m)_{m=1}^n$ the event history of the process indicating, for each single event $m$, its type $u_m$ and time of occurrence $t_m$.
Then, the non-negative stochastic rate of occurrence of each $N_u(t)$ is defined by:
\begin{equation}\label{eq:hawkes}
\lambda_u(t) = \mu_u(t) + \sum_{m : t_m < t} A_{u_m u}\,g_{u_m u}(t - t_m).
\end{equation}
In the above, $\mu_u(t) \op{\geq} 0$ is the \emph{natural occurrence rate} of events of type $u$ (\ie along that dimension) at time $t$, and the \emph{triggering kernel function} evaluation $g_{vu}(s \op{-} t) \geq 0$ determines the \emph{increase} in the occurrence rate of events of type $u$ at time $s$, caused by an event of type $v$ at a past time $t \op{\leq} s$.

The natural occurrence rates $\mu_u$ and triggering kernels $g_{vu}$ are usually inferred by means of log-likelihood maximization. The main practical issue for inferring the parameters of the model in \Eq{eq:hawkes} is that it requires a particularly large dataset of observations, and standard inference algorithms require at least one observation per pair of event types (\ie $d^2$ observations). In many practical situations, the underlying network of interactions is unknown, and in such a case, we will consider that each event type can be affected by any other, hence $A_{uv} \op{=} 1$ for every pair of event types $u \op{\neq} v$. An index of the main notations used in this paper is provided in \Tab{tab:notations}.
%
%
%
\section{\methodName}\label{sec:mmh} 
%
%
\subsection{The proposed model}\label{sec:model}
%
%
%

\inlinetitle{Model considerations}{.} 
%
%
Standard \mhp inference requires the learning  of $d^2$ triggering kernels that encode the cross- and self-excitement of the event types. This requirement becomes prohibitive when $d$ is very large (\eg when the dimensions represent the users of a social network or websites on the Internet). However, in a number of practical situations, the $d^2$ complex interactions between event types can be summarized by considering that there is a small number of $r$ groups to which each event type belongs to a certain extent. Therefore, one needs to simultaneously learn a $d \op{\times} r$ event type-to-group(s) mapping (we specifically use \emph{soft} assignments) as well as the $r^2$ interactions between pairs of those groups.


\medskip
\inlinetitle{Model formulation}{.} 
\emph{\methodName} (\methodInitials) simplify the standard inference process by projecting the original $d$ event types (\ie dimensions) of a multivariate Hawkes process into a smaller and more compact $r$-dimensional space. The natural occurrence rates $\mu_u$ and triggering kernels $g_{vu}$ of \Eq{eq:hawkes} are then defined via the low-rank approximation:
\begin{equation}\label{eq:mu-and-g}
\begin{array}{l}
\ \mu_u(t) = \sum_{i=1}^r P_{ui} \, \tilde{\mu}_i(t),\\\\
g_{vu}(t) = \sum_{i,j=1}^r P_{ui} \, P_{vj} \, \tilde{g}_{ji}(t),
\end{array}
\end{equation}
where $u, v$ are event types, $P \op{\in} \matdef{\mathbb{R}_{+}}{d}{r}$ is the projection matrix from the original $d$-dimensional space to the low-dimensional space, and $i, j$ are its component directions. Besides, this projection can be seen as a low-rank approximation of the kernel matrix $g$ since, in matrix notations, $g \op{=} P\tilde{g}P^\top$ and $\tilde{g} \op{\in} \matdef{\mathbb{R}_{+}}{r}{r}$ is a matrix of size $r \op{\ll} d$.

Then, the \methodInitials occurrence rates are formulated as an extension of \Eq{eq:hawkes} that uses an embedding of event types in a low-dimensional space:
\begin{equation}\label{eq:intensityValue}
\begin{split}
\!\lambda_u(t) = \sum_{i=1}^r P_{ui} \, \tilde{\mu}_i(t) \\
&\hspace{-20mm} + \sum_{m : t_m < t}\ \sum_{i,j=1}^r P_{ui} \, P_{u_m j} \, A_{u_m u} \, \tilde{g}_{ji}(t - t_m).
\end{split}
\end{equation}
Specifically, if the projection of event type $u$ along the dimension $i$ is given by $P_{ui}$, then the event type $u$ essentially inherits the natural occurrence rate of events of that component $\tilde{\mu}_i$, with multiplicative weight $P_{ui}$, that is $\sum_{i=1}^r P_{ui} \tilde{\mu}_i$. In addition, if the projection of event type $v$ along each dimension $j$ is given by $P_{vj}$, then $v$'s effect on event type $u$ will be evaluated by $\sum_{i,j=1}^r P_{ui} P_{vj} \tilde{g}_{ji}$. 

Keeping in mind that $r \op{\ll} d$, the proposed low-rank approximation is a simple and straightforward way to: i)~impose regularity to the inferred occurrence rates by introducing constraints to the parameters, and ii)~reduce the number of parameters. Specifically, 
the $d$ natural rates and $d^2$ triggering kernels are reduced to $r$ and $r^2$, respectively, with the only extra need of inferring the $d \op{\times} r$ elements of the matrix $P$.

\medskip
\inlinetitle{Remark on the uniqueness of the projection}{.} 
Unless any further assumption is made on the projection matrix $P$ or the low-dimensional kernel $\tilde{g}$, the low-rank decomposition of the triggering kernel $g \op{=} P\tilde{g} P^\top$ is not unique. More specifically, any change of basis in the $r$-dimensional space will not alter the decomposition. Notwithstanding, \emph{uniqueness} is not required in order to perform the prediction task, and therefore we do not address this issue in the present paper.
\subsection{Log-likelihood}
\label{sec:loglik}
%
\inlinetitle{General formulation}{.} 
For $h \op{=} 1,\mydots,H$, let $\mathcal{H}^h = (t_m^h, u_m^h)_{m\leq n_h}$ be the observed \iid realizations sampled from the Hawkes process, and $\mathcal{H} = (\mathcal{H}^h)_{h\leq H}$ the recorded history of events of all realizations. For each realization $h$, we denote as $[T_-^h, T_+^h]$ the observation period, and $u_m^h$ and $t_m^h$ are respectively the event type and time of occurrence of the $m$-th event. 
Then, the log-likelihood of the observations can be written as:
\begin{equation} \label{eq:general_loglik}
\begin{split}
\mathcal{L}(P,&\mathcal{H};\mu,g)=
%
\sum_{h=1}^H \left[
\sum_{m=1}^{n_h} \ln \left(
\sum_{i=1}^r P_{u_m^h i} \, \tilde{\mu}_i(t_m^h) \right. \right. \\
&\hspace{-6mm} \left. + \sum_{i,j} \sum_{l :\,t_l^h < t_m^h}\! P_{u_m^h i} \, P_{u_l^h j} \, A_{u_l^h u_m^h} \, \tilde{g}_{ji}(t_m^h - t_l^h) \right) \\
&\hspace{-6mm} - \sum_{u,i} P_{ui} \int_{T_-^h}^{T_+^h} \tilde{\mu}_i(s) ds \\
&\hspace{-6mm} - \left. \! \sum_{u,v,i,j} P_{ui} \, P_{vj} \, A_{vu} \int_{T_-^h}^{T_+^h} \tilde{g}_{ji}(s - t_m^h) ds \right]\!\!.
%
\end{split}
\end{equation}

Our objective is to infer the natural rates $\tilde{\mu}_i$ and triggering kernels $\tilde{g}_{ji}$ by means of log-likelihood maximization. 
From \Eq{eq:general_loglik}, we see that, for arbitrary $\tilde{g}_{ji}$, a single log-likelihood computation already necessitates $O(\sum_{h=1}^{H} n_h^2)$ triggering kernel evaluations. This is intractable when individual realizations can have a number of events of the order $10^7$ or $10^8$ (\eg a viral video when modeling information cascades). This issue can be tackled by relying on a convenient $K$-approximation introduced in \cite{lemonnier2014nonparametric}. Each natural occurrence rate and kernel function are approximated by a sum of $K$ exponential triggering functions:
\begin{equation}
\begin{array}{l}
\widehat{\mu}_i^K(t) = \sum_{k=0}^K \beta_{i, k}\, e^{-k \gamma t},\\\\
\widehat{g}_{ji}^K(t) = \sum_{k=1}^K \alpha_{ji, k}\, e^{-k \delta t},
\end{array}
\end{equation}
where $\gamma,\delta > 0$ are fixed hyperparameter values.

Due to the \emph{memoryless property} of exponential functions, this approximation allows for log-likelihood computations with complexity linear in the number of events, \ie $O(n = \sum_{h=1}^{H} n_h)$. Results of polynomial approximation theory  also ensures fast convergence of the optimal $\widehat{\mu}_i^K$ and $\widehat{g}_{ji}^K$ towards the \emph{true} $\tilde{\mu}_i$ and $\tilde{g}_{ji}$ with respect to $K$. For instance, if $\tilde{g}_{ji}$ is analytic, then $\sup_{t \in [0,T]} | \widehat{g}_{ji}^K(t) \op{-} \tilde{g}_{ji}(t)| \op{=} O(e^{-K})$ which means that, for smooth enough functions, choosing $K \op{=} 10$ already provides a good approximation.

We therefore search the values of parameters $\alpha,\beta$ 
that will maximize the approximated log-likelihood, as well as the most probable projection matrix $P$, conditionally to the realizations of the process, and under the constraint that the approximated natural rates and triggering kernels remain non-negative. At high-level, this is formally expressed as:
\begin{equation}
\begin{array}{c}
\displaystyle\argmax_{(P, \alpha, \beta
)}\ \widehat{\mathcal{L}}(P, \mathcal{H};\alpha,\beta)\\
\st \forall i,j,t,K : \ \ \widehat{\mu}_i^K(t) \geq 0 \mbox{ \ and \ } \widehat{g}_{ji}^K(t) \geq 0.
\end{array}
\end{equation}
Above, for clarity of notation, we actually reformulate the log-likelihood by introducing $\widehat{\mathcal{L}}$ that makes implicit the dependency of $\mathcal{L}$ in the fixed hyperparameters $K$, $\delta$, and $\gamma$. Note also that limiting $K$ and $r$ to small values can be seen as a form of regularization, although more refined approaches could be considered in case of training with datasets of very limited size.

\label{sec:matloglik}

\medskip
\inlinetitle{Simplification with tensor notation}{.} 
In order to perform inference efficiently, we now reformulate the log-likelihood using very large and sparse tensors. We also introduce the artificial $(r \op{+} 1)$-th dimension to the embedding space in order to remove linear terms of the equation and store the $\beta$ parameters as additional dimensions of $\alpha$. In detail, let $\alpha_{(r+1)i,k} \op{=} \beta_{i,k}$, $\alpha_{j(r+1),k} \op{=} 0$, and $P_{(d+1)i} \op{=} \one_{\{i=r+1\}}$ (note that $\one_{\{\cdot\}}$ denotes the indicator function), also, $\forall u \op{\in} \{1,\mydots,d\}$, $P_{u(r+1)} \op{=} 0$. Now, the log-likelihood of the model can be rewritten in the following way:
\begin{equation}
\begin{split}
\widehat{\mathcal{L}}(P, &\mathcal{H};\alpha) = \sum_{h,m} \ln \left(
\sum_{u,v,i,j,k} P_{ui} \, P_{vj} \, \alpha_{ji,k} \, D_{h,m,u,v,k}
\right)\\
&\hspace{12mm} -\sum_{h,u,v,i,j,k} P_{ui} \, P_{vj} \, \alpha_{ji,k} \, B_{h,u,v,k},
\end{split}
\end{equation}
%
where
%
\begin{equation} \label{eq:Bmatrix}
B_{h,u,v,k} = \left\{
\begin{array}{ll}
\sum_{m=1}^{n_h} J_{v,u,m} f_{k \delta}(T_{+}^h-t_m^h) &\mbox{if } v \op{\leq} d;\\
f_{k \gamma}(T_{+}^h - T_{-}^h) &\mbox{if } v \op{=} d \op{+} 1;\\
0 &\mbox{otherwise},
\end{array}\right.
\end{equation}
%
%
\begin{equation}\label{eq:Dmatrix}
D_{h,m,u,v,k} = \left\{
\begin{array}{llll}
\sum_{l=1}^{n_h} I_{h,m,l,u,v} \, e^{-k\delta (t_m^h - t_l^h)} &\mbox{if } v \op{\leq} d;\\
\one_{\{u_m^h = u\}} e^{-k\gamma (t_m^h - T_-^h)} &\mbox{if } v \op{=} d \op{+} 1;\\
0 &\mbox{otherwise},
\end{array}\right.
\end{equation}
with
\vspace{-1em}
\begin{itemize}
\item[] \hspace{0.4em} $\displaystyle f_{k x}(t) = \frac{1 - e^{-k x T}}{k x}$,\ \ for $x$ in $\{\gamma,\delta\}$,
\vspace{-0.0em}
\item[] \hspace{0.3em} $J_{v,u,m} = \one_{\{v=u_m^h\}} A_{vu}$,
\vspace{-0.0em}
\item[] \hspace{-0.8em} $I_{h,m,l,u,v} = \one_{\{u=u_m^h \ \wedge\  v=u_l^h \ \wedge\ t_l^h < t_m^h\}} A_{vu}$.\\
\vspace{-0.5em}
\end{itemize}
%
What is suggested by the expressions is the possibility to optimize the approximated log-likelihood, according to the different parameters and projection matrices, by first creating two large and sparse tensors $B$ and $D$ with four and five dimensions, respectively. 


\begin{algorithm}[tb!]
   \caption{Inference: high-level description}
   \label{alg:main}
\begin{algorithmic}
	 \STATE \textbf{Input:} history of events $\mathcal{H}$; hyperparameters $K$, $\gamma$, $\delta$; initialized projection matrix $P$ and 
	triggering kernel parameters $\alpha$
	 \STATE Compute $D$ and $B$ \hfill \small{//\,\emph{see \Alg{alg:construct_B_D}}}
	 \FOR{$i = 1$ \TO \emph{num$\_$iters}}
	 \STATE $\alpha = \argmax_{\alpha}\ \widehat{\mathcal{L}}(P,\mathcal{H};\alpha)$
     \STATE $\quad\ \ $ s.t. $\widehat{\mu}_i^K \geq 0$ and $\widehat{g}_{ji}^K \geq 0$,\  \,$i,j \op{=} 1,\mydots,r$
	 \STATE $P = \argmax_P\ \widehat{\mathcal{L}}(P,\mathcal{H};\alpha)$
	 \ENDFOR
	 \RETURN $P,\alpha$
\end{algorithmic}
\end{algorithm}

%
%
%
\section{The inference algorithm}\label{sec:inference}
%
%
The inference is performed by alternating optimization between the projection matrix $P$ and the Hawkes parameters $\alpha$. When all others parameters are fixed, the optimization \wrt $\alpha$ is performed using self-concordant function optimization with self-concordant barriers. The technical difficulty of this part is due to the need to ensure that non-negativity constraints are respected. For the optimization \wrt $P$, we introduce new optimization techniques based on a minorize-maximization algorithm. \Alg{alg:main} outlines the general scheme of our optimization algorithm. Recall that our basic notation is indexed in \Tab{tab:notations}.


\medskip
\inlinetitle{Computing $B$ and $D$ tensors}{.}
In order for the inference algorithm to be tractable, special attention has to be paid to the computation of $B$ and $D$ tensors. 
\Alg {alg:construct_B_D} describes the computation of the sparse tensors $B=(B_{h,u,v,k})$ and $D=(D_{h,m,u,v,k})$.
The most expensive operation in this algorithm is the multiplicative update of all $C_{v}^{k}$ with the exponential decay $\exp(-(k \op{+} \one_{\left\{v>0\right\}}) \gamma dt)$. Fortunately, this update only has to be performed for every node $v$ that already appeared in the cascade, which are at most $\sigma \leq d$ (by definition). The complexity of this operation is therefore $O(nK\sigma)$. The number of non-zero elements of $D$ and $B$ is $O(nK \min(\Delta,\sigma))$, where $\Delta$ is the maximum number of neighbors of a node in the underlying network $\mathcal{G}$.
If $\mathcal{G}$ is sparse, which is usually the case for social networks for instance, then $\Delta \op{\ll} d$ and therefore $O(nK\Delta) \op{\ll} O(nKd)$. Thus, 
storing and computing $B$ and $D$ is tractable for large dense graphs and for particularly large sparse graphs.
Note that, because computing the log-likelihood requires the computation of occurrence rates at each event time, which depends on the occurrences of all preceding events, the linear complexity in the number of events is only possible because of the memoryless property of the decomposition over a basis of exponentials. Otherwise, the respective complexity would have been at least $\Theta(\sum_{h=1}^{H} n_h^2 K\sigma)$, with $\sum_{h=1}^H n_h^2 \op{\gg} n$.


\begin{algorithm}[t] \small
\caption{Construction of $D$ and $B$ tensors}
\begin{algorithmic}\label{alg:construct_B_D}
\STATE {Initialize $j=0$}
\FORALL{$h$}
\STATE{Initialize $(C_{v}^{k}=\one_{\left\{v=d+1\right\}})_{v \geq 0, k \geq 0}$ ; $t_0^h=T_-^h$} ; $(B'_{h,u,k} = 0)_{u \geq 0, k \geq 0}$
\STATE{$B'_{h,d+1,k} \leftarrow \frac{1-\exp(-k \gamma (T_+^h - T_-^h))}{k\gamma}$}
\FORALL{$m \in [1\mydots n_h]$}
\STATE{$dt \leftarrow t_m^h - t_{m-1}^h$}
\FORALL{$k$,$v$ \mbox{s.t} $C_{v}^{k}>0$}
\STATE{$C_{v}^{k} \leftarrow C_{v}^{k} \exp(-\one_{\left\{v>0\right\}}(k+1) \delta dt - \one_{\left\{v=0\right\}}\gamma dt)$}
\ENDFOR
\FORALL{$k$}
\STATE{$D_{h,m,u,v,k} \leftarrow \one_{\left\{u=u_m\right\}} \sum_{v \geq 0}
A_{u_{m}v}C_v^{k}$}
\STATE{$B'_{h,u_{m},k} \leftarrow B'_{h,u_{m},k} + \frac{1-\exp(-k \delta (T_+^h - t_m^h))}{k\delta}$}
\STATE{$C_{u_m}^{k} \leftarrow C_{u_m}^{k} + 1$}
\ENDFOR
\STATE{$j \leftarrow j+1$}
\ENDFOR
\STATE{$B_{h,u,v,k} \leftarrow A_{uv} B'_{h,v,k}$}
\ENDFOR
\RETURN $B$, $D$
\end{algorithmic}
\end{algorithm}
%

\medskip
\inlinetitle{Hawkes parameters optimization}{.}
Updating the Hawkes parameters $\alpha$ requires solving the problem:
\begin{equation}
\begin{array}{c}
\displaystyle\alpha = \argmax_\alpha\ \sum_{h,m} \ln \left({c^{hm}}^\top \alpha\right) - b^\top \alpha \\
\mbox{s.t } \widehat{\mu}_i^K \geq 0 \mbox{ and } \widehat{g}_{ji}^K \geq 0, \ i,j=1,\mydots,r,
\end{array}
\end{equation}
where 
\vspace{-1em}
\begin{itemize}
\item[] \hspace{0.6em} $c^{hm}_{ijk} = \sum_{u,v} P_{ui} \, P_{vj} \, D_{h,m,u,v,k}$, \\
\vspace{-1.5em}
\item[] \hspace{0.7em} $b_{ijk} = \sum_{u,v,h} \, P_{ui} \, P_{vj} \, B_{h,u,v,k}$. 
\end{itemize}
\vspace{-0.5em}
For the sake of inference tractability, we relax the non-negativity constraint and only impose it for the observed time differences: 
\begin{equation}
\sum_{k=1}^{K} \alpha_{ji,k} \, D_{h,m,u,v,k}\ \geq\ 0.
\end{equation}
Then, we approximate the constrained maximization problem by an unconstrained one, using the concept of \emph{self-concordant barriers} \cite{nesterov1994interior}. More specifically, we choose $\epsilon \op{>} 0$ and solve:
\begin{equation}
\label{eq:unconstrainedHawkes}
\alpha = \argmax_\alpha\ \sum_{h,m} \left(\ln \left({c^{hm}}^\top \alpha\right) + \epsilon b(\alpha) \right) - b^\top \alpha,
\end{equation}
where
\begin{equation}
b(\alpha)_{hm} =\! \sum_{i,j,u,v} \ln \left(\sum_{k=1}^{K} \alpha_{ji,k} \, D_{h,m,u,v,k} \right)\!.
\end{equation}
A feature of the optimization problem in \Eq{eq:unconstrainedHawkes} is that it verifies the \emph{self-concordance property}. Self-concordant functions have the advantage of behaving nicely with barrier optimization methods and are among the rare classes of functions for which explicit convergence rates of Newton methods are known \cite{boyd2004convex}. This is the reason why we chose to perform the unconstrained optimization using Newton's method, which requires $O(nKr^2 \op{+} K^3 r^6)$ operations. Note that, since we have $n$ events and aim to learn $K$ Hawkes parameters per pair of groups, we have necessarily $Kr^2 \op{\ll} n$. If we do not have $K^2 r^4 \ll n$, we can reduce the complexity by using quasi-Newton methods that necessitates only $O(nKr^2 \op{+} K^2 r^4) \op{=} O(nKr^2)$ operations. The computation of $c$, $b$ and $b(\alpha)$ requires multiplying sparse matrices of $O(nK\Delta)$ non-zero elements with a full matrix of $r$ columns, which yields a $O(nK\Delta r)$ complexity. Overall, the complexity of the Hawkes parameters optimization is of the order $O(nKr(\Delta \op{+} r))$.

\medskip
\inlinetitle{Projection matrix optimization}{.}
Let $p$ a reshaping of the projection matrix $P$ to a vector (linearized), then $p$ is updated by solving the following maximization procedure:
\begin{equation}
p = \argmax_p\ \sum_{h,m} \ln \left(p^\top \Xi^{hm} p\right) - p^\top \Psi p,
\end{equation}
where 
\vspace{-1em}
\begin{itemize}
\item[] \hspace{-0.7em} $2\,\Xi_{ui,vj}^{hm} = \sum_k (\alpha_{ji,k} D_{h,m,u,v,k} + \alpha_{ij,k} D_{h,m,v,u,k})$, \\
\vspace{-1.5em}
\item[] \hspace{-0.8em} $2\,\Psi_{ui,vj} = \sum_{h,k} (\alpha_{ji,k} B_{h,u,v,k} + \alpha_{ij,k} B_{h,v,u,k})$. \\
\end{itemize}
\vspace{-1.5em}
%
The maximization task is performed by a novel minorize-maximization procedure which is summarized by the following proposition, proved in the Appendix.
\begin{proposition}\label{prop:projopt}
The log-likelihood is non-decreasing under the update:
\begin{equation}\label{eq:projopt}
\begin{split}
p_{ui}^{t+1} = 
& p_{ui}^t \left(\sum_{h,m} \frac{(\Xi^{hm}p^t)_{ui}}{{p^t}^\top \Xi^{hm} p^t(\Psi p^t)_{ui}}\right)^{1/2}. 
\end{split}
\end{equation}
Furthermore, if $p_{ui}$ is a stable fixed point of \Eq{eq:projopt}, then $p_{ui}$ is a local maximum of the log-likelihood.
\end{proposition}
As previously, the computation of $\Xi$, $\Psi$, and all the matrix-vector products requires $O(nK\Delta r^2)$ operations, and each update necessitates $O(nd)$ operations. Again, we consider settings where we have at least a few events per dimension, so the total complexity of the group affinities optimization is $O(nK\Delta r^2)$.

In total, the complexity of the whole optimization procedure is of the order $O(nK \sigma + nK\Delta r^2)$ and its behavior is linear \wrt the number of events and the number of dimensions.

\begin{figure}[t!]
	\centering
	\includegraphics[width=0.6\linewidth, viewport=1 1 233 200, clip]{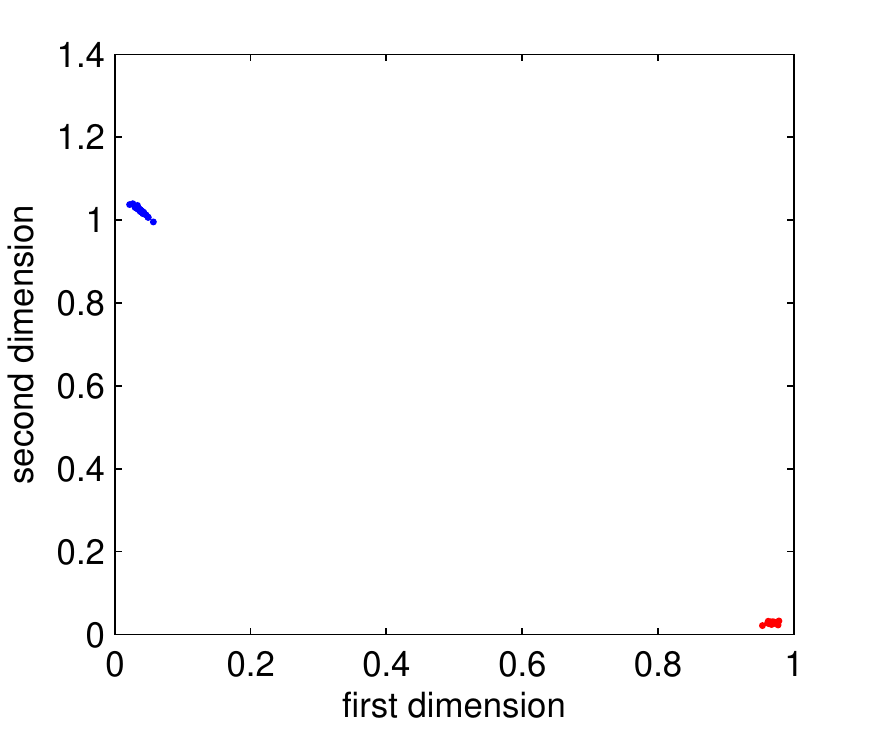}
\caption{Low-dimensional embedding of the event types learned by \methodInitials in the synthetic dataset. The two groups (blue and red) of event types are successfully identified.}
\label{fig:embeddings}
\end{figure}

%


%
%
%
\section{Experiments}\label{sec:exps}
%
For testing the performance of the proposed \methodInitials model and the efficiency of our inference algorithm, the experimental study consists of two parts. First, we simulate MHPs on small random networks and verify that the parameters of the simulation are recovered by our algorithm. Second, we provide results on a prediction task for the MemeTracker dataset in order to show that: i)~\methodInitials is highly competitive compared to state-of-the-art inference algorithms on medium-sized datasets, and ii)~\methodInitials is the first framework able to perform large-scale inference for MHPs.
%
%
%
\subsection{Synthetic data }\label{sec:exps_toy}
%
In this section we illustrate the validity and precision of our method in learning the diffusion parameters of simulated Hawkes processes.
More specifically, we simulate MHPs such that event types are separated into two groups of similar activation pattern. In the context of social networks, these groups may encode \emph{influencer-influencee} types of relations. We show that our inference algorithm can recover the groups and corresponding triggering kernels consistently and with high accuracy. Note that \methodInitials is more generic than this setting, however, we believe that such simple scenario may provide a clearer overview of the capabilities of our approach.

\begin{figure}[t!]
\begin{center}
\includegraphics[width=0.9\linewidth]{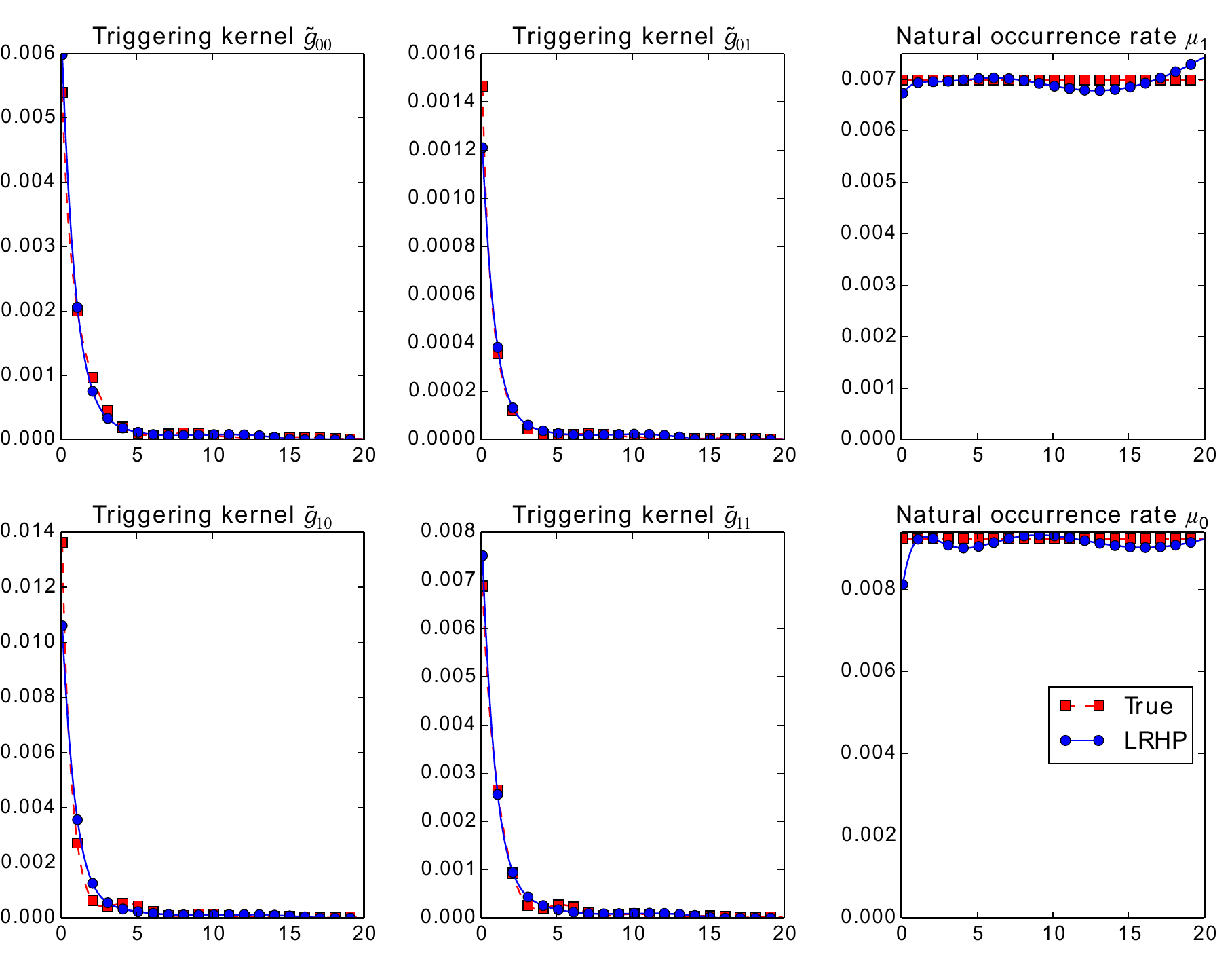}
\caption{True and inferred triggering kernels $\tilde{g}_{ij}$ and natural occurrence rates $\tilde{\mu}_i$, for the synthetic dataset.}
\label{fig:kernels}
\end{center}
\end{figure}

\medskip
\inlinetitle{Data generation procedure}{.}
The employed procedure for generation of synthetic datasets is as follows. 
%
We assume that the MHPs take place on a random \Erdos \cite{erdHos1976evolution} network of $d \op{=} 100$ event types whose adjacency matrix $A$ is generated with parameter $p \op{=} 0.1$ (\ie 10 neighbors in average). 
Then, we consider two distinct groups of event types, and assign each event type to one of the groups at random. 
The natural occurrence rate $\tilde{\mu}_i$ of each group is fixed to a constant value chosen uniformly over $[0,0.01]$. The triggering kernels between two groups, $i$ and $j$, are generated as:
\begin{equation}
\tilde{g}_{ij}(t)=\nu_{ij}\frac{\sin \left(\frac{2\pi t}{\omega_{ij}}+\frac{\pi}{2}((i+j) \bmod 2)\right)+2}{3(t+1)^2},
\label{eqn:TrigKerData}
\end{equation}
where $\omega_{ij}$ and $\nu_{ij}$ are sampled uniformly over respectively $[1,10]$ and $[0,1/50]$, respectively. These parameter intervals are chosen so that the behavior of the generated process is \emph{non-explosive} \cite{daley2007introduction}.
The rationale behind the kernels in \Eq{eqn:TrigKerData} is that they present a power-law decreasing intensity that allows long term influence with a periodic behavior. This kind of dynamics could, for instance, represent the daytime cycles of internet users.

\medskip
\inlinetitle{Results}{.}
Following the above procedure we generate 8 datasets by sampling 8 different sets of parameters \{$(\omega_{ij},\nu_{ij})_{i\leq r, j \op{\leq} r}$, $(\tilde{\mu}_i)_{i\leq r}$\}. Finally, we simulate $10^5$ \iid realizations of the resulting Hawkes process, that we use as training set. The ability of \methodInitials to recover the true group triggering kernels $\tilde{g}_{ij}$, is shown in \Fig{fig:kernels} and evaluated by means of the \emph{normalized $L^2$ error}:
\begin{equation}
\frac{1}{r^2}\sum_{i,j} \frac{||\widehat{g}_{ij} - \tilde{g}_{ij}||_2}{||\widehat{g}_{ij}||_2 + ||\tilde{g}_{ij}||_2}.
\end{equation}
In average, this is only $4.2\%$, with minimum and maximum value amongst the 8 datasets of $3.8\%$ and $4.7\%$, respectively.

In order to find the group assignments, we infer the parameters of an \methodInitials of rank $r \op{=} 2$, and recover the group structure by a clustering algorithm on the projected event types.
Then, choosing as basis of the two-dimensional space the centers of the two clusters enables the recovery of the group triggering kernels.
\Fig{fig:embeddings} shows the two-dimensional embedding learned by our inference algorithm for one of the 8 sample datasets. Two particularly separate clusters appear, which indicates that the group assignments were perfectly recovered. The other 7 datasets gave similar results.
Moreover \Fig{fig:kernels} compares visually the fitness of the inferred to the true natural occurrence rates and triggering kernel functions. 

These results provide strong indication regarding the validity of our algorithm for inferring the underlying dynamics of MHPs.
\begin{table*}[t] \footnotesize
\caption{Experiments on the MemeTracker datasets. \emph{AUC} (\%) and \emph{Accuracy} (\%) for predicting the \emph{next event to happen}, using \methodInitials, MEMIP, and NAIVE approach. In each case, the CPU time (secs) needed for training is also reported. The experiments for the missing measurements, denoted with `$*$', did not finish in reasonable time.}
\vspace{2mm}
\hspace{-4em}
\begin{tabular}{lrrr | r r | r r r | r r r}
\toprule
\multicolumn{4}{c|}{\textbf{\emph{Dataset}}} & \multicolumn{2}{c|}{\textbf{\emph{Training Time}}} & \multicolumn{3}{c|}{\textbf{\emph{AUC}}} & \multicolumn{3}{c}{\textbf{\emph{Accuracy}}}\\
Name & \emph{thd} & $n$ & $d$ & \methodInitials & MEMIP & \methodInitials & MEMIP & NAIVE & \methodInitials & MEMIP & NAIVE \\
\midrule
MT$_1$  & 50000 & 7311   & 13   & $8.34$ & $3.16$ & $86.4$ & $85.8$ & $86.1$ & $98.8$ & $99.2$ & $93.1$ \\
MT$_2$  & 10000 & 74474  & 80   & $281$  & $7.14 \op{\cdot} 10^3$ & $90.1$ & $91.7$ & $84.4$ & $91.7$ & $93.7$ & $70.6$ \\ 
MT$_3$  & 5000  & 277914 & 172  & $1.95 \op{\cdot} 10^3$ & $1.74 \op{\cdot} 10^5$ & $84.3$ & $85.9$ & $81.2$ & $88.6$ & $91.1$ & $67.7$ \\
MT$_4$  & 1000  & 875402 & 1075 & $3.77 \op{\cdot} 10^5$ & * & $92.6$ & * & $88.2$ & $94.8$ & * & $87.5$ \\
\bottomrule
\end{tabular}
\label{table:PredSimul}
\end{table*}

\begin{figure}[t!]
\begin{center}
\includegraphics[width=0.6\linewidth, viewport=1 1 306 262, clip]{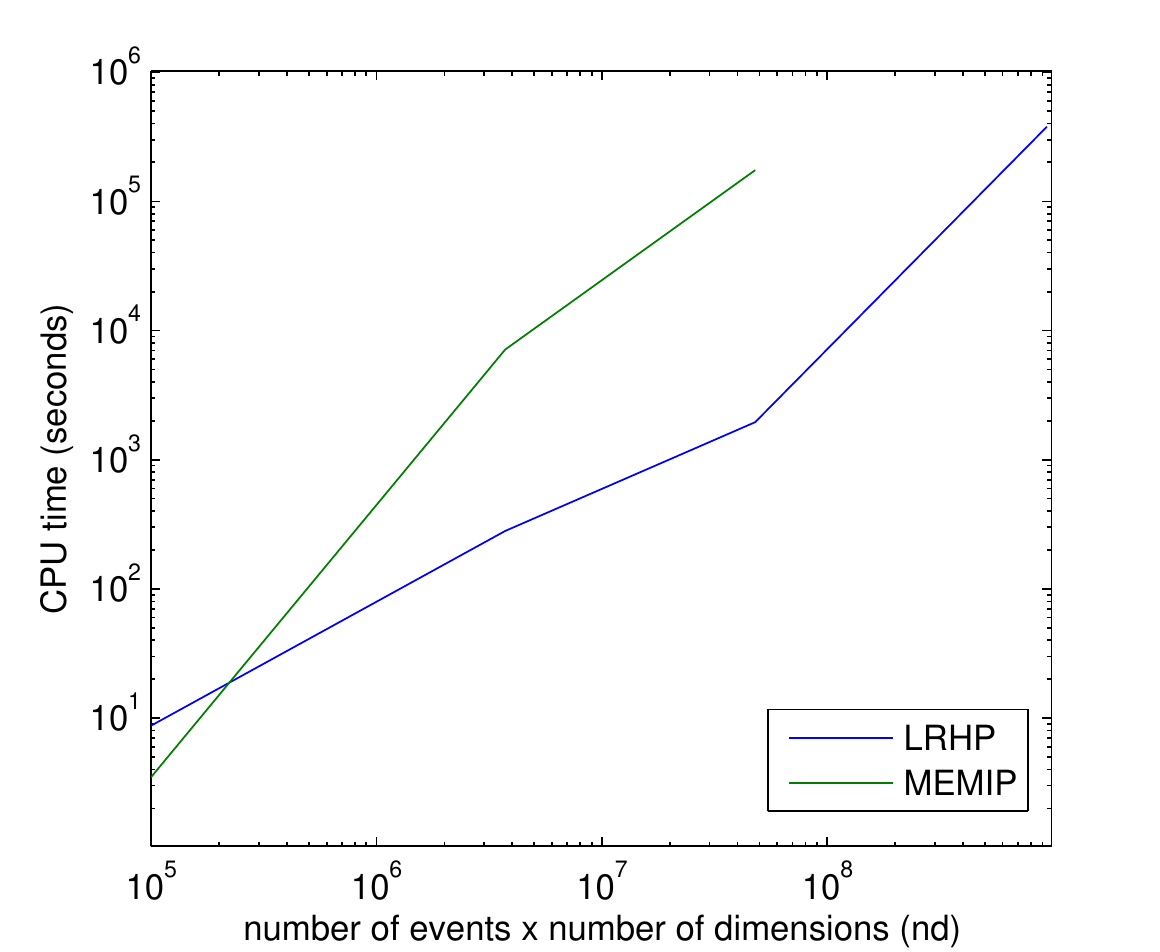}
\caption{Training time (secs) for \methodInitials and MEMIP algorithm against the quantity $nd$. The linear behavior for \methodInitials and super-linear for MEMIP are clearly visible.}
\label{fig:cputime}
\end{center}
\end{figure}

%
%
%
%
\subsection{Results on the MemeTracker dataset}\label{sec:exps_memtracker}
%
Our final set of experiments are conducted on the MemeTracker \cite{snapnets} dataset. MemeTracker is a corpus of $96 \op{\cdot} 10^5$ blog posts published between August 2008 and April 2009. We use posts from the period August 2008 to December 2008 as training set, and evaluate our models on the four remaining months. An \emph{event for website $u$} is defined as the creation of a hyperlink on website $u$ towards any other website. 
We also consider that an edge exists between two websites if at least one hyperlink exists between them in the training set.
In order to compare the inference algorithms on datasets of different size, prediction was performed on four subsets of the MemeTracker dataset: MT$_1$, MT$_2$, MT$_3$ and MT$_4$. These subsets are created by removing the events taking place on websites that appear less than a fixed number of times in the training set. This threshold value ($thd$ in \Tab{table:PredSimul}) is, respectively, $50000$, $10000$, $5000$ and $1000$.

\medskip
\inlinetitle{Prediction task}{.}
The task consists in predicting the \emph{next website to create a post}. More specifically, for each event of the test dataset, we are interested in predicting the website on which it will take place knowing its time of occurrence. 
For MEMIP and \methodInitials, prediction will be achieved by scoring the websites according to $\lambda_u(t_m)$, since this value is proportional to the theoretical conditional probability for event $m$ to be of type $u$.
We evaluate the prediction with two metrics: the area under the ROC curve (AUC) and a classification accuracy with a fixed number of candidate types. Due to the high bias towards major news websites (\eg CNN), the number of candidate types has to be relatively large to see differences in the performance of algorithms, and we set this value to $30\%$ of the total number of event types $d$ in our experiments.

\medskip
\inlinetitle{Baselines}{.}
In the following experiments, we use as main competitor the state-of-the-art MEMIP algorithm \cite{lemonnier2014nonparametric}, which is, to the best of our knowledge, the only inference algorithm with linear complexity in the number of events $n$ in the training history. Also, previous work \cite{lemonnier2014nonparametric} shows that this algorithm outperforms the more standard inference algorithm MMEL \cite{conf/icml/ZhouZS13} on the MemeTracker dataset.
In addition, we also use the NAIVE baseline which ranks the nodes according to their frequency of appearance in the training set. Note that this is equivalent to fitting a Poisson process and, hence, does not consider mutual-excitation.
%

%

\medskip
\inlinetitle{Results}{.}
\Tab{table:PredSimul} summarizes the experimental results comparing the proposed \methodInitials against MEMIP and NAIVE algorithms on four subsets of the MemeTracker dataset. In each row, the table describes the dataset characteristics, and for each method it provides the training time, AUC, and accuracy with the best parameter settings (for \methodInitials, $K\op{=}6$ and $r\op{=}2$, except for MT$_3$ for which $r\op{=}3$). On small to medium-sized datasets (MT$_1$, MT$_2$ and MT$_3$), \methodInitials is as efficient as its main competitor MEMIP, while orders of magnitude faster. On the large dataset MT$_4$, \methodInitials still runs in reasonable time while substantially outperforming the NAIVE baseline. Note that MEMIP could not be computed in reasonable time for this dataset (less than a few days). 

\Fig{fig:cputime} shows the computational time needed for the inference algorithm on all the MemeTracker datasets, with respect to $nd$. This time is indeed linear in $nd$ for \methodInitials, while super-linear for the state-of-the-art competitor of the related literature. In \Fig{fig:MMHPSensiTracker} it is indicated that the accuracy measurements are relatively stable \wrt the rank of the approximation $r$, with a maximum for $r \op{=} 3$. Finally, \Fig{fig:embeddingsMemeTracker} shows the two-dimensional embedding learned by \methodInitials for the MT$_3$ dataset. In the embedding space, the websites seem to align along the axes of the embedding space, with varying amplitudes. This may indicate that two basic groups of similar activities were recovered by the algorithm, although with a large variability in the activity of the websites.

\begin{figure}[t!]
\begin{center}
\includegraphics[width=0.6\linewidth, viewport=3 0 231 181, clip]{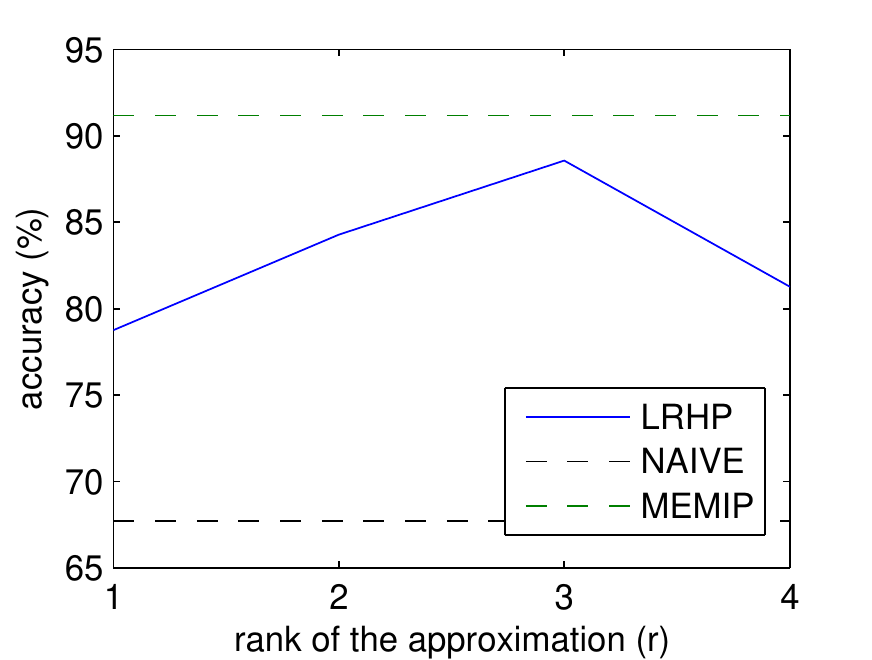}
\vspace{-2mm}
\caption{Sensitivity analysis of the accuracy of \methodInitials \wrt the rank $r$ of the approximation used for inference, and a comparison to the best scores for MMEL and Naive baselines on the MT$_3$ dataset.}
\label{fig:MMHPSensiTracker}
\end{center}
\end{figure}

\begin{figure}[t!]
	\centering
	\includegraphics[width=0.6\linewidth, viewport=1 1 233 200, clip]{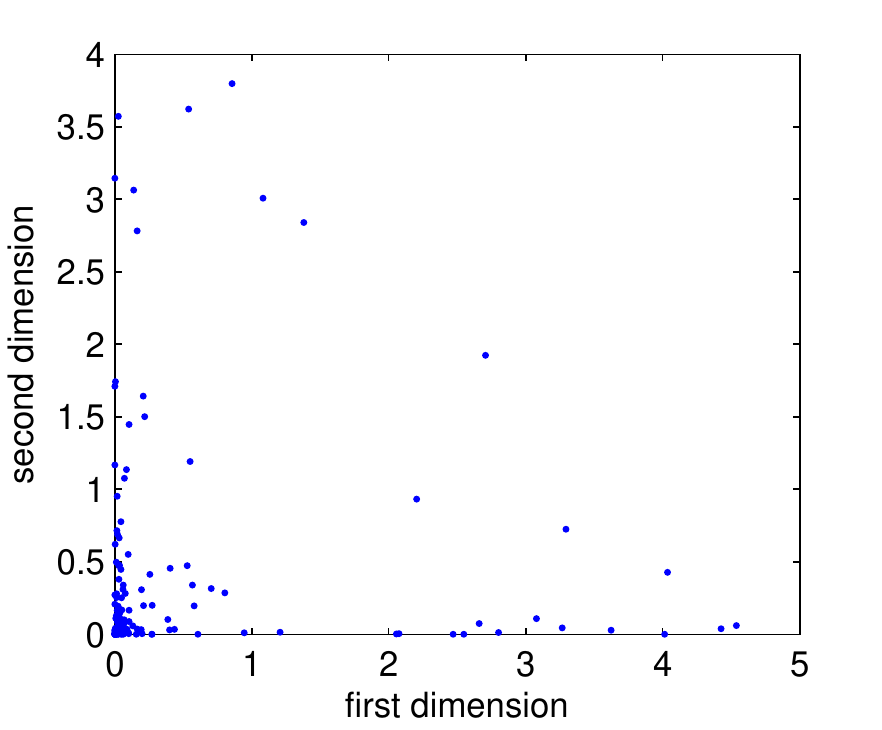}
	\vspace{-2mm}
\caption{Low-dimensional embeddings of the event types learned by \methodInitials for the MT$_3$ dataset.
}
\label{fig:embeddingsMemeTracker}
\end{figure}
%
%
%
\section{Conclusion}\label{sec:conclusion}
%
This work focused on modeling multivariate time series where both a very large number of event types can occur, and a very large number of historical observations are available for training. We introduce a model based on multivariate Hawkes processes that we call \emph{\methodName} (\methodInitials), and develop an inference algorithm for parameter estimation. Theoretical complexity analysis as well as experimental results show that our approach is highly scalable, while performing as efficiently as state-of-the-art inference algorithms in terms of prediction accuracy.

\subsection*{Appendix: Proof of \Proposition{prop:projopt}.}

For this proof we will make use of the concept of \emph{auxiliary functions}.
\begin {definition}
Let $g\!\!: \mathcal{X}^2 \op{\rightarrow} \mathcal{R}$ is an auxiliary function for  $f\!\!: \mathcal{X} \op{\rightarrow} \mathcal{R}$ iff $\forall (x,y) \op{\in} \mathcal{X}^2,\\
g(x,y) \op{\geq} f(x)$ and $\forall x \op{\in} \mathcal{X}, g(x,x) \op{=} f(x)$.
\end{definition}
The reason why these functions are an important tool for deriving iterative optimization algorithms is given by the following lemma.
\begin {lemma}\label{lem:auxfundef}
If $g$ is an auxiliary function for $f$, then 
\begin{equation}
f\left(\arg\!\min_{x} g(x,y) \right) \leq f(y). 
\end{equation}
\end{lemma}
\begin{proof}
Let $z \op{=} \arg\!\min_{x} g(x,y)$. Then $$ f(z) = g(z,z) \leq g(z,y) \leq g(y,y) = f(y).$$
where the first inequality comes from the definition of $g$ and the second from the definition of $z$.
\end{proof}
Therefore, if an auxiliary function $g$ is available, constructing the sequence $y_{t+1} = \arg\!\min_{x} g(x,y_{t})$ that verifies $f(y_{t+1}) \op{\leq} f(y_{t})$ for all $t$ constitutes a candidate method for finding the minimum of $f$. In our case, we are able to make use of the following result.
\begin {lemma}\label{lem:auxfun}
Let $f(p) \op{=} - \sum_{k=1}^K \ln \left(p^\top \Xi^k p\right) + p^\top \Psi p$ 
where $p\in\mathbb{R}_+^K$, $\Xi^1$, ..., $\Xi^K$
are positive symmetric matrices and $\Psi$ is a symmetric matrix, then
\begin{equation}
\begin{split}
g(p,q) = & - \sum_{k=1}^K \left(\frac{2q^\top \Xi^k [q \ln(p/q)]}{q^\top \Xi^k q} + \ln \left(q^\top \Xi^k q\right)\right) \\ 
& + q^\top \Psi [p^2/q]
\end{split}
\end{equation}
is an auxiliary function for $f$.
\end{lemma}
In the lemma above, the vectors $[q \ln(p/q)]$ and $[p^2/q]$ are to be understood as coordinate-wise operations, \ie $(q_{i}\ln(p_{i}/q_{i}))_{i} $ and $(p_{i}^{2}/q_{i})_{i}.$
\begin{proof}
It is clear that $g(p,p) \op{=} f(p)$ so the proof reduces to showing that $g(p,q) \op{\geq} f(p)$. Let $k\leq K$. By concavity of the logarithm function, we have for every weight matrix $(\alpha_{ij})_{ij}$ such that $\sum_{i,j} \alpha_{ij} \op{=} 1 $, 
$$\ln \left(p^\top \Xi^{k} p\right) \geq \sum_{i,j} \alpha_{ij} \ln \left(\frac{p_{i} \Xi^{k}_{ij} p_{j}}{\alpha_{ij}} \right).$$ Note that the right-hand side term of the equation is well-defined because of the positivity constraint imposed on each $\Xi^{k}_{ij}$. By choosing 
$ \alpha_{ij} = q_{i} \Xi^{k}_{ij} q_{j}/q^\top \Xi^{k} q$, and using the symmetry of $\Xi^{k}$, we get:
$$ \ln \left(p^\top \Xi^{k} p\right) \geq \frac{2q^\top \Xi^{k} [q \ln(p/q)]}{q^\top \Xi^{k} q} + \ln \left(q^\top \Xi^{k} q\right).$$ For the right-hand side of the above equation, we use the fact that for every $i,j$ it holds $$p_{i}p_{j} \leq \frac{p_{i}^2 q_{j}}{2q_{i}} + \frac{p_{j}^2 q_{i}}{2q_{j}},$$ and the symmetry of $\Psi$, in order to conclude that $p^\top \Psi p \op{\leq} q^\top \Psi [p^2/q].$
\end{proof}
Using \Lemma{lem:auxfun}, we are now in position to prove \Proposition{prop:projopt} by showing that the proposed update $p^{t+1}$ is indeed the global minimum of $g(p,p^{t})$. $g$ being the sum of univariate convex functions of the $p_{i}$, it is sufficient to show that for every $i$, the partial derivative of $g(p,p^{t})$ with respect to $p_{i}$ vanishes in $p^{t+1}_{i}$. We therefore need:
$$ - \sum_{k}  \frac{p^t_{i} (\Xi^{k}p^t)_{i}}{p^{t+1}_{i}{p^t}^\top \Xi^{k} p^t} + \frac{p^{t+1}_{i}(\Psi p^t)_{i}}{p^t_{i}}=0,$$
which only positive solution is given by:
\begin{equation}\label{eq:lasteq}
p_{i}^{t+1} = 
p_{i}^t \left(\sum_{k} \frac{(\Xi^{k}p^t)_{i}}{{p^t}^\top \Xi^{k} p^t(\Psi p^t)_{i}}\right)^{1/2}.
\end{equation}
Finally, if $p$ is a stable fixed point of \Eq{eq:lasteq}, then, by definition, there exists $\epsilon \op{>} 0$ such that, $\forall p'$ for which $||p \op{-} p'||_2 \op{\leq} \epsilon$, the iterative algorithm starting at $p^0 \op{=} p'$ converges to $p$. However, since $f$ is continuous, a simple iteration of the inequality of \Lemma{lem:auxfundef} implies that $f(p') \op{\geq} f(p^1) \op{\geq} \mydots \op{\geq} \lim_{t\rightarrow +\infty} f(p^t) \op{=} f(p)$, and $p$ is a local minimum of $f$.

\bibliographystyle{unsrt}
\bibliography{LargeScaleHawkes}


\end{document}

%% file: definitions.tex
\usepackage{graphics}

\usepackage{algorithm}
\usepackage{algorithmic}
\usepackage{enumitem}

\usepackage{amsthm}
\newtheorem{definition}{Definition}

\newtheorem{lemma}{Lemma}

\newtheorem{proposition}{Proposition}

\usepackage{url}
\usepackage{amsmath}

\usepackage{amsfonts}
\usepackage{dsfont}
\usepackage{upgreek}

\usepackage{graphicx}
\graphicspath{ {./figures/} } 

\usepackage{subfigure}

\usepackage{booktabs} 

\usepackage{xspace}

\newcommand{\Fig}[1]		{Fig.\,\ref{#1}}

\newcommand{\Eq}[1]			{Eq.\,\ref{#1}}
\newcommand{\Tab}[1]		{Tab.\,\ref{#1}}
\newcommand{\Alg}[1]		{Alg.\,\ref{#1}}
\newcommand{\Lemma}[1]		{Lemma\,\ref{#1}}
\newcommand{\Proposition}[1]		{Proposition\,\ref{#1}}
\newcommand{\ie}   			{i.e.~}
\newcommand{\eg}   			{e.g.~}
\newcommand{\wrt}   		{w.r.t.~}
\newcommand{\iid}   		{i.i.d.~}

\newcommand{\st}   			{\mbox{s.t.~}}

\newcommand{\Erdos}   	{Erd\"os-R\'enyi }
\newcommand{\one}       {\mathds{1}}

\newcommand{\op}[1]     {\,{#1}\,}
\newcommand{\matdef}[3] {#1^{#2 \times #3}}  

\newcommand{\inlinetitle}[2]  {\noindent\textbf{\emph{#1}{#2}}}

\renewcommand{\footnotesize}   {\small}

\DeclareMathOperator*{\argmax}{arg\,max}
\newcommand{\mydots}   			{{...}}